\documentclass[journal,10pt,twocolumn]{IEEEtran}

\usepackage{amsmath}
\usepackage{amssymb}
\usepackage{amsthm}
\usepackage{graphicx}
\usepackage{hyperref}
\usepackage{bm}
\usepackage{bbm}
\usepackage{booktabs}
\usepackage{tikz}
\usetikzlibrary{arrows.meta,positioning}

\newtheorem{proposition}{Proposition}
\newtheorem{definition}{Definition}

\newtheorem{conjecture}{Conjecture}

\begin{document}

\title{Kernel Dynamics under Path Entropy Maximization}

\author{%
  \IEEEauthorblockN{Jnaneshwar Das}
  \IEEEauthorblockA{%
    School of Earth and Space Exploration\\
    Arizona State University, Tempe, AZ 85287, USA\\
    Earth Innovation Hub\\
    \texttt{jnaneshwar.das@asu.edu}%
  }%
}

\maketitle

\begin{abstract}
We propose a variational framework in which the kernel function
$k : \mathcal{X} \times \mathcal{X} \to \mathbb{R}$, interpreted as
the foundational object encoding what distinctions an agent can represent,
is treated as a dynamical variable subject to path entropy maximization
(Maximum Caliber, MaxCal). Each kernel defines a representational structure
over which an information geometry on probability space may be analyzed;
a trajectory through kernel space therefore corresponds to a trajectory
through a family of effective geometries, making the optimization landscape
endogenous to its own traversal. We formulate fixed-point conditions for
self-consistent kernels, propose renormalization group (RG) flow as a
structured special case, and suggest neural tangent kernel (NTK) evolution
during deep network training as a candidate empirical instantiation.
Under explicit information-thermodynamic assumptions, the work required
for kernel change is bounded below by $\delta W \geq k_B T \, \delta I_k$,
where $\delta I_k$ is the mutual information newly unlocked by the updated
kernel. In this view, stable fixed points of MaxCal over kernels correspond
to self-reinforcing distinction structures, with biological niches,
scientific paradigms, and craft mastery offered as conjectural
interpretations. We situate the framework relative to assembly theory and
the MaxCal literature, separate formal results from structured
correspondences and conjectural bridges, and pose six open questions that
make the program empirically and mathematically testable.
\end{abstract}


% -----------------------------------------------------------------------
\section{Introduction}
% -----------------------------------------------------------------------

Every inference engine---biological, computational, or learned---operates
on a substrate that determines which differences in the world it can
represent. We call this substrate the \emph{kernel}: a positive-definite
function $k : \mathcal{X} \times \mathcal{X} \to \mathbb{R}$ whose
choice induces an inner product on a reproducing kernel Hilbert space
(RKHS) $\mathcal{H}_k$, and---via the representational structure it
defines---an effective geometry on the space of probability distributions
$\mathcal{P}$ and a bound on extractable work via the Sagawa-Ueda
generalized second law~\cite{sagawa2010}.

Throughout the paper, ``kernel'' is used in a broadened but still
technical sense: not merely as a computational device for nonlinear
learning, but as the mathematical object that determines which differences
in the world are representable, comparable, or actionable for an agent.
This representational reading motivates treating kernel change as a
change in the agent's effective distinction-making capacity, which is
the quantity the MaxCal construction is designed to model.

In this paper we adopt the kernel as the primitive object from which
geometry, dynamics, and information-thermodynamic bounds are derived.
At fixed kernel, this static viewpoint recovers a common exchange rate
$k_B \ln 2$ per bit under standard assumptions of reversible bookkeeping.
The present manuscript focuses on what that static viewpoint leaves open:
a given agent operates \emph{within} a kernel but the question of how
kernels themselves change over time was deferred.

The present paper addresses that deferral directly. We ask: \emph{if
kernels are dynamical variables, what variational principle governs their
trajectories?}

The answer we develop is Maximum Caliber (MaxCal)~\cite{press2013}---path
entropy maximization over trajectories through kernel space. The resulting
framework has three properties that distinguish it from prior treatments
of kernel learning:

\begin{enumerate}
    \item The optimization landscape is \emph{endogenous}: each kernel
    $k$ deforms the geometry of $\mathcal{P}$, so the landscape
    through which kernels evolve is itself a function of the trajectory.

    \item Fixed points of the dynamics are \emph{self-consistent} kernels:
    distinction structures that are self-reinforcing under their own
    effective geometry. We interpret RG fixed points as the most
    structurally direct special case, and biological niches, scientific
    paradigms, and craft mastery as candidate higher-level analogues of
    self-consistent kernels.

    \item Kernel change has a \emph{thermodynamic cost} bounded below by
    $k_B T$ per bit of newly unlocked mutual information---a
    Landauer principle for conceptual change.
\end{enumerate}

Relative to prior work, the contribution here is not a new
kernel-learning algorithm, nor a reformulation of classical information
geometry, nor a standard MaxCal model on a fixed state space. Rather,
the paper treats the kernel itself as the evolving object of inference
and asks what variational principle governs trajectories through the
space of distinction-making structures. The result is a framework that
is partly formal and partly programmatic: it introduces the kernel-space
MaxCal construction, formulates self-consistency and stability conditions,
and uses these to organize a set of candidate correspondences across
physics, learning, biology, and embodied craft.

The paper is organized as follows. Section~\ref{sec:background} reviews
the necessary background in MaxCal, RKHS geometry, and the
information-thermodynamic assumptions used here.
Section~\ref{sec:kernel_space} defines the space of kernels and its
topology. Section~\ref{sec:maxcal_kernels} lifts MaxCal from state space
to kernel space. Section~\ref{sec:fixed_points} derives fixed-point
conditions and their stability. Section~\ref{sec:special_cases}
instantiates the framework across RG flow, NTK evolution, adaptive
sample-return planning, biological evolution, and craft mastery.
Section~\ref{sec:assembly} situates the framework relative to assembly
theory. Section~\ref{sec:claim_levels} separates formal results from
conjectural bridges. Section~\ref{sec:open} poses six open questions.

\noindent\textbf{Claim status in this paper.}
\emph{Formal content} includes the definition of kernel space
$\mathcal{K}$, the lifting of MaxCal to path measures on kernel
trajectories, the self-consistency condition, and the frozen-kernel
stability criterion.
\emph{Structured correspondences} include the mappings to RG flow and
finite-width NTK evolution.
\emph{Conjectural bridges} include the interpretations in biological
evolution, craft mastery, adaptive field sampling, and scientific
paradigm shifts.

% -----------------------------------------------------------------------
\section{Background}
\label{sec:background}
% -----------------------------------------------------------------------

\subsection{Maximum Caliber}

Let $\Gamma$ denote the space of trajectories $\gamma : [0,T]
\to \mathcal{X}$ over a state space $\mathcal{X}$. MaxCal selects
the path distribution $p[\gamma]$ that maximizes the path entropy
\begin{equation}
    \mathcal{S}[p] = -\sum_\gamma p[\gamma] \ln \frac{p[\gamma]}{q[\gamma]}
    \label{eq:maxcal}
\end{equation}
subject to dynamical constraints $\langle f_i[\gamma] \rangle = F_i$,
where $q[\gamma]$ is a reference path measure~\cite{press2013}.
The resulting distribution takes the form
\begin{equation}
    p[\gamma] \propto q[\gamma] \exp\!\left(-\sum_i \lambda_i f_i[\gamma]\right)
\end{equation}
where $\lambda_i$ are Lagrange multipliers. MaxEnt (Jaynes) is
recovered when trajectories reduce to single configurations. MaxCal
derives Green-Kubo relations, Onsager reciprocity, Prigogine's minimum
entropy production, and master equations as special cases~\cite{press2013}.

\subsection{RKHS Geometry and the Kernel Primitive}

A Mercer kernel $k$ induces an RKHS $\mathcal{H}_k$ via the
feature map $\phi : \mathcal{X} \to \mathcal{H}_k$,
$\phi(x) = k(\cdot, x)$. The kernel is treated as the primitive
representational object, while Fisher-Rao supplies the canonical
information geometry on the associated statistical manifold; kernel
change therefore alters the effective geometry of inference by altering
the representational substrate on which that manifold is defined.
Concretely, the square-root embedding
$p \mapsto \sqrt{p} \in L^2(\mathcal{X}, \nu)$ pulls back the
$L^2$ inner product to the Fisher-Rao metric~\cite{amari2016}
\begin{equation}
    g^F_{ij}(p) = \int \frac{\partial \ln p}{\partial \theta^i}
    \frac{\partial \ln p}{\partial \theta^j} \, p \, d\nu
\end{equation}
The Hellinger kernel $k_H(p,q) = \int \sqrt{p \, q} \, d\nu$ is the
unique kernel whose induced geometry respects sufficient
statistics~\cite{chentsov1982}. The kernelized Stein discrepancy identifies
the score function $\nabla_x \log p(x)$ as the Riesz representative
of the Stein operator in $\mathcal{H}_k$~\cite{liu2016}.

The central modeling assumption of this paper is that $k$ is prior to all
of this structure: the metric, the score function, the mutual information
$I$, and the Sagawa-Ueda work bound
\begin{equation}
    W_{\text{extracted}} \leq \Delta F + k_B T \cdot I
    \label{eq:sagawa_ueda}
\end{equation}
are treated as derived objects once a kernel is fixed.

\subsection{Standing Assumptions}

The following assumptions are used throughout the paper.
\begin{enumerate}
    \item[\textbf{(A1)}] $\mathcal{X}$ is a Polish (complete, separable,
    metrizable) space.

    \item[\textbf{(A2)}] $\nu$ is a $\sigma$-finite reference measure on
    $\mathcal{X}$.

    \item[\textbf{(A3)}] Every kernel $k \in \mathcal{K}$ is Mercer
    (continuous, symmetric, positive semi-definite) and Hilbert-Schmidt:
    $\iint k(x,x')^2 \, d\nu(x) \, d\nu(x') < \infty$.

    \item[\textbf{(A4)}] The agent--environment system admits a joint
    distribution $P_{\mathrm{agent,env}}$ from which the mutual information
    $I_k = I(A;E \mid k)$ is computed by restricting the agent's
    representation to the RKHS $\mathcal{H}_k$.

    \item[\textbf{(A5)}] Thermodynamic bookkeeping is quasi-static: the
    Landauer bound $\delta W \geq k_B T \, \delta I$ (in nats) is
    attainable in the reversible limit.

    \item[\textbf{(A6)}] All information quantities ($I$, $D_{\mathrm{KL}}$)
    are measured in nats unless stated otherwise, so that Landauer's
    bound reads $k_B T$ per nat (equivalently $k_B T \ln 2$ per bit).
\end{enumerate}

% -----------------------------------------------------------------------
\section{The Space of Kernels}
\label{sec:kernel_space}
% -----------------------------------------------------------------------

\begin{definition}[Kernel space]
Let $\mathcal{K}$ denote the set of all Mercer kernels on
$\mathcal{X}$:
\[
    \mathcal{K} = \{ k : \mathcal{X} \times \mathcal{X} \to \mathbb{R}
    \mid k \text{ symmetric, positive semi-definite} \}
\]
\end{definition}

$\mathcal{K}$ is a convex cone: if $k_1, k_2 \in \mathcal{K}$ and
$\alpha, \beta \geq 0$, then $\alpha k_1 + \beta k_2 \in
\mathcal{K}$. Products $k_1 \cdot k_2 \in \mathcal{K}$ as well,
making $\mathcal{K}$ closed under the operations that arise naturally
in kernel composition.

A natural metric on $\mathcal{K}$ is induced by the
Hilbert-Schmidt norm on the associated integral operators
$T_k : L^2(\nu) \to L^2(\nu)$:
\begin{equation}
\begin{split}
    d(k_1, k_2) &= \| T_{k_1} - T_{k_2} \|_{\text{HS}} \\
    &= \biggl( \int\!\int [k_1(x,x') - k_2(x,x')]^2 \\
    &\qquad\qquad\quad d\nu(x) \, d\nu(x') \biggr)^{1/2}
\end{split}
\end{equation}
With this metric, $\mathcal{K}$ is a separable metric space and
paths $\gamma : [0,T] \to \mathcal{K}$ are well-defined as
Bochner-integrable trajectories.

\subsection{The Endogenous Landscape}

The critical structural feature of $\mathcal{K}$ is that each point
$k \in \mathcal{K}$ determines a representational substrate relative to
which an effective information geometry on $\mathcal{P}$ may be analyzed.
A path $\gamma(t) \in \mathcal{K}$ therefore
generates a one-parameter family of Riemannian manifolds
$(\mathcal{P}, g_{\gamma(t)})$. The landscape over which the kernel
evolves is itself a function of the kernel's current value.
Writing $g_k$ for the metric induced by kernel $k$, the chain rule gives
\begin{equation}
    \frac{d}{dt} g_{\gamma(t)} = Dg\big|_{\gamma(t)}[\dot\gamma(t)]
\end{equation}
where $Dg|_k : T_k\mathcal{K} \to \mathrm{Sym}^2(T^*\mathcal{P})$ is
the Fr\'echet derivative of the map $k \mapsto g_k$ at $k$, evaluated in
direction $\dot\gamma(t)$.
This endogeneity---the landscape depends on the current position---distinguishes
kernel dynamics from ordinary optimization on a fixed manifold.

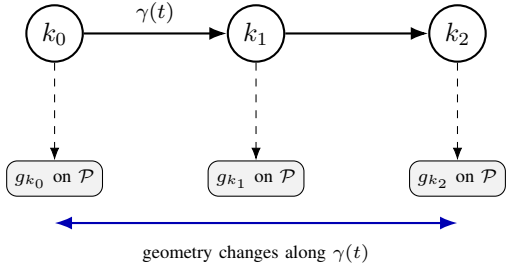
\begin{figure}[t]
\centering
\begin{tikzpicture}[
    >=Latex,
    node distance=1.9cm,
    kernel/.style={circle,draw,thick,minimum size=7.5mm,inner sep=0pt},
    labelbox/.style={draw,rounded corners,fill=gray!10,inner sep=3pt,font=\scriptsize}
]
    \node[kernel] (k0) {$k_0$};
    \node[kernel,right=of k0] (k1) {$k_1$};
    \node[kernel,right=of k1] (k2) {$k_2$};

    \draw[->,thick] (k0) -- node[above,font=\footnotesize] {$\gamma(t)$} (k1);
    \draw[->,thick] (k1) -- (k2);

    \node[labelbox,below=1.3cm of k0] (g0) {$g_{k_0}$ on $\mathcal{P}$};
    \node[labelbox,below=1.3cm of k1] (g1) {$g_{k_1}$ on $\mathcal{P}$};
    \node[labelbox,below=1.3cm of k2] (g2) {$g_{k_2}$ on $\mathcal{P}$};

    \draw[dashed,->] (k0) -- (g0);
    \draw[dashed,->] (k1) -- (g1);
    \draw[dashed,->] (k2) -- (g2);

    \draw[<->,thick,blue!70!black]
    ([yshift=-3.5mm]g0.south) -- ([yshift=-3.5mm]g2.south)
    node[midway,below=1.5mm,font=\scriptsize,black]
    {geometry changes along $\gamma(t)$};
\end{tikzpicture}
\caption{Illustration of a kernel trajectory $\gamma(t)$ in $\mathcal{K}$.
Each kernel $k_t$ determines a representational substrate relative to
which a metric $g_{k_t}$ on probability space $\mathcal{P}$ is analyzed,
so moving through $\mathcal{K}$ simultaneously deforms the geometry on
which inference proceeds.}
\label{fig:kernel_trajectory}
\end{figure}

% -----------------------------------------------------------------------
\section{MaxCal over Kernel Space}
\label{sec:maxcal_kernels}
% -----------------------------------------------------------------------

We now lift MaxCal from $\mathcal{X}$ to $\mathcal{K}$. Let
$\Pi$ denote the space of paths $\gamma : [0,T] \to \mathcal{K}$,
and let $\mathcal{Q}$ be a reference measure on $\Pi$.
By~(A3), $\mathcal{K}$ embeds isometrically into the separable Hilbert
space of Hilbert-Schmidt operators on $L^2(\nu)$; a natural choice
for $\mathcal{Q}$ is the Wiener measure on this Hilbert space,
conditioned on the positive-semidefinite cone.

\begin{definition}[Kernel path entropy]
The path entropy over kernel trajectories is
\[
    \mathcal{S}[P] = -\int_\Pi P[\gamma]
    \ln \frac{P[\gamma]}{\mathcal{Q}[\gamma]} \, \mathcal{D}\gamma
\]
where $P$ is a probability measure on $\Pi$.
\end{definition}

\subsection{A Minimal Two-Kernel Toy Model}
\label{sec:toy_model}

To make the framework concrete, consider a finite kernel family
$\mathcal{K}_2 = \{k_A, k_B\}$ and discrete time
$t = 0,1,\dots,T$. Let a trajectory be
$\gamma = (k_{t=0},\dots,k_{t=T})$ with $k_t \in \mathcal{K}_2$.
Choose a Markov reference process
\begin{equation}
    \mathcal{Q}[\gamma] = \pi_0(k_0)\prod_{t=0}^{T-1} q(k_{t+1}\mid k_t).
\end{equation}
Impose two path constraints: expected cumulative switching cost
\begin{equation}
    C[\gamma] = \sum_{t=0}^{T-1}\mathbbm{1}[k_{t+1}\neq k_t],
\end{equation}
and expected cumulative information gain
\begin{equation}
    G[\gamma] = \sum_{t=0}^{T} I_{k_t}.
\end{equation}
The MaxCal distribution is then
\begin{equation}
    P[\gamma] \propto \mathcal{Q}[\gamma]
    \exp\!\left(-\lambda_C C[\gamma] + \lambda_G G[\gamma]\right).
\end{equation}
This yields an effective two-state dynamics with transition odds
\begin{equation}
    \frac{P(k_{t+1}=k_B\mid k_t=k_A)}{P(k_{t+1}=k_A\mid k_t=k_A)}
    \propto
    \frac{q(B\mid A)}{q(A\mid A)}\exp\!\left(-\lambda_C+\lambda_G\Delta I\right),
\end{equation}
where $\Delta I = I_{k_B}-I_{k_A}$.
The toy model exhibits an explicit threshold:
switching is favored when $\lambda_G\Delta I > \lambda_C$.
This is the simplest instantiation of the tradeoff between
thermodynamic/transition cost and informational gain.

\subsection{Constraints}

Physically meaningful constraints on kernel trajectories include:

\begin{enumerate}
    \item \textbf{Thermodynamic cost constraint.}
    The expected work required to shift the kernel along a path must not
    exceed available free energy:
    \[
        \mathbb{E}_P\!\left[\int_0^T \dot{W}_k(t) \, dt\right] \leq \mathcal{F}
    \]
    where $\dot{W}_k$ is the instantaneous thermodynamic cost of
    kernel change (derived below in Section~\ref{sec:cost}).

    \item \textbf{Fidelity constraint.}
    The kernel must maintain sufficient mutual information about the
    environment's relevant microstates:
    \[
        \mathbb{E}_P[I_{k(t)}] \geq I_{\min}(t)
    \]

    \item \textbf{Consistency constraint.}
    The kernel must remain consistent with the agent's current
    generative model. Writing $p_{\mathrm{env}}$ for the environmental
    data distribution and $q_{k}$ for the agent's model distribution
    under kernel $k$ (see A4):
    \[
        \mathbb{E}_P\!\left[D_{\mathrm{KL}}(p_{\mathrm{env}} \| q_{k(t)})\right] \leq \epsilon
    \]
\end{enumerate}

\subsection{The MaxCal Kernel Distribution}

Maximizing $\mathcal{S}[P]$ subject to these constraints yields
\begin{multline}
    P[\gamma] \propto \mathcal{Q}[\gamma] \,
    \exp\!\Bigl(-\lambda_1 \int_0^T \dot{W}_k \, dt \\
    + \lambda_2 \int_0^T I_{k(t)} \, dt
    - \lambda_3 \int_0^T D_{\mathrm{KL}}(t) \, dt \Bigr)
    \label{eq:kernel_maxcal}
\end{multline}
This is the \emph{least-assuming} distribution over kernel trajectories
consistent with thermodynamic, informational, and consistency constraints.
The mode of $P$ is the most probable kernel trajectory under the
chosen reference measure and constraints, and may be interpreted as the
least-assuming path of distinction-change available to the agent model
under those assumptions.

\subsection{Thermodynamic Cost of Kernel Change}
\label{sec:cost}

The cost of moving from kernel $k$ to $k + \delta k$ is the
work required to update the agent's distinction-making capacity.
Creating mutual information between agent and environment requires
physical work by Landauer's principle~\cite{landauer1961}; the
Sagawa-Ueda bound~(\ref{eq:sagawa_ueda}) quantifies the
complementary direction (work extraction from existing correlations).
Together they imply that acquiring mutual information
$\delta I_k$ that was inaccessible to $k$ but accessible to
$k + \delta k$ costs at minimum (in nats, per assumption~A6):
\begin{equation}
    \delta W_k \geq k_B T \cdot \delta I_k
    \label{eq:kernel_landauer}
\end{equation}
where
\begin{equation}
    \delta I_k = I_{k+\delta k}(A; E)
    - I_k(A; E)
\end{equation}
with $A$, $E$ the agent and environment variables of~(A4).
Equation~(\ref{eq:kernel_landauer}) may be read as a
\emph{Landauer-type lower bound for physically realized kernel change}
under assumptions~(A4)--(A6): gains in discriminative capacity that are
physically realized by the agent carry a lower-bounded thermodynamic
cost.
Combining~(\ref{eq:kernel_landauer}) with the chain rule
$\dot{I}_k = \langle \nabla_k I_k, \dot{k}\rangle_{\mathrm{HS}}$
and Cauchy-Schwarz gives a speed limit on kernel evolution:
\begin{equation}
    \bigl\| \dot{k}(t) \bigr\|_{\mathrm{HS}} \leq
    \frac{\dot{\mathcal{F}}(t)}{k_B T \cdot
    \| \nabla_k I_k \|_{\mathrm{HS}}}
\end{equation}
where $\nabla_k I_k$ is the Hilbert-Schmidt gradient of $I$ with
respect to $k$, and $\dot{\mathcal{F}}$ is the rate of free-energy
supply. Kernels can only change as fast as free energy permits.

\noindent\textit{Scope of the thermodynamic bound.}
The bound~(\ref{eq:kernel_landauer}) should be read as an
information-thermodynamic lower bound under assumptions~(A4)--(A6),
not as a claim that arbitrary abstract conceptual change automatically
admits a direct calorimetric interpretation independent of embodiment.
The intended claim is narrower: when a change in kernel corresponds to
a physically realized increase in the agent's representational access
to environmental distinctions, Landauer-style reasoning yields a lower
bound on the work required to acquire that additional information.
The Sagawa-Ueda relation is invoked as the complementary result
governing extraction from existing correlations, not as a substitute
for the acquisition bound itself.

% -----------------------------------------------------------------------
\section{Fixed Points and Their Stability}
\label{sec:fixed_points}
% -----------------------------------------------------------------------

\begin{definition}[Self-consistent kernel]
A kernel $k^* \in \mathcal{K}$ is \emph{self-consistent} if it is a
fixed point of the MaxCal kernel dynamics:
\[
    \frac{\delta \mathcal{S}}{\delta \gamma}\bigg|_{\gamma \equiv k^*} = 0
\]
That is, $k^*$ is the maximum-caliber choice given the geometry that
$k^*$ itself induces.
\end{definition}

\subsection{Stability Criterion}

Stability of $k^*$ concerns whether small perturbations
$k^* + \epsilon\,h$ (with $h \in T_{k^*}\mathcal{K}$) grow or decay.
Let $\mathcal{S}^*(k)$ denote the optimized path entropy when the
kernel is held fixed at $k$ (i.e., $\mathcal{S}^*(k) =
\max_P \mathcal{S}[P]$ subject to the constraints of
Section~\ref{sec:maxcal_kernels}, with the kernel frozen at $k$).
A self-consistent kernel $k^*$ is \emph{stable} if the Hessian
of the self-consistency map satisfies
\begin{equation}
    D^2_k \mathcal{S}^*\big|_{k^*}[h,h] < 0 \quad
    \text{for all } h \in T_{k^*}\mathcal{K},\; h \neq 0.
\end{equation}
Unstable directions correspond to
bifurcations---transitions between distinct distinction regimes.
In the two-kernel model of Section~\ref{sec:toy_model}, stability
reduces to whether perturbations in the effective gain-cost balance
$\lambda_G \Delta I - \lambda_C$ return the system to the same
occupancy regime or drive a transition to the alternative kernel; this
illustrates in finite dimensions the broader interpretation of unstable
directions as transition channels between distinction regimes.

\begin{conjecture}
The set of stable self-consistent kernels forms a discrete (generically
zero-dimensional) subset of $\mathcal{K}$, separated by unstable
fixed points that act as transition states between basins of attraction.
\end{conjecture}

The basins of attraction of stable self-consistent kernels are the precise
mathematical content of what we informally call \emph{niches},
\emph{paradigms}, and \emph{mastery domains}.

% -----------------------------------------------------------------------
\section{Special Cases}
\label{sec:special_cases}
% -----------------------------------------------------------------------

\subsection{Renormalization Group Flow as Kernel Dynamics}

In Wilson's renormalization group~\cite{wilson1974}, integrating out
degrees of freedom below a momentum cutoff $\Lambda$ defines a map
$R_\Lambda : \mathcal{K} \to \mathcal{K}$. The RG flow is the
trajectory
\begin{equation}
    \frac{dk}{d\ln\Lambda} = \beta(k)
\end{equation}
where $\beta$ is the beta function. Fixed points of RG flow satisfy
$\beta(k^*) = 0$---scale-invariant kernels at which the distinction
between microscale and macroscale collapses.

\begin{proposition}
RG flow can be represented as a special case of MaxCal over kernels in which:
(i) the constraint is scale invariance of the partition function, and
(ii) the reference measure $\mathcal{Q}$ is uniform over the
renormalization group orbit.
\end{proposition}

\noindent\textit{Proof status.}
This statement is currently a structural correspondence rather than a full
derivation. A complete proof requires an explicit map between the RG
coarse-graining semigroup and the MaxCal path measure on $\Pi$.

\begin{conjecture}
The critical exponents at RG fixed points equal the eigenvalues of
$D^2_k \mathcal{S}^*|_{k^*}$ in the Hilbert-Schmidt metric, so that
universality classes correspond to stability basins of self-consistent kernels.
\end{conjecture}

\subsection{Neural Tangent Kernel Evolution}

For an infinitely wide neural network, the neural tangent kernel (NTK)
$\Theta(x, x')$ governs the training dynamics~\cite{jacot2018}:
\begin{equation}
    \frac{d\hat{y}}{dt} = -\Theta \cdot \nabla_{\hat{y}} \mathcal{L}
\end{equation}
During training of finite networks, $\Theta$ evolves. The MaxCal
kernel framework predicts:

\begin{enumerate}
    \item The trajectory of $\Theta(t)$ through $\mathcal{K}$
    maximizes path entropy subject to cross-entropy loss reduction at rate
    $\geq \dot{\mathcal{L}}_{\min}$.

    \item Fixed points of NTK evolution correspond to networks that have
    learned a self-consistent representation of their training
    distribution---feature hierarchies that are self-reinforcing.

    \item The thermodynamic cost bound~(\ref{eq:kernel_landauer}) predicts
    a minimum energy dissipation during NTK evolution:
    $\dot{W} \geq k_B T \cdot \dot{I}_\Theta$, which may be
    empirically testable via GPU power consumption during training.
\end{enumerate}

\begin{conjecture}
The NTK of a trained diffusion model converges to the Hellinger kernel
$k_H$ of the data distribution---the unique kernel whose geometry
respects sufficient statistics~\cite{chentsov1982}.
\end{conjecture}

In the NTK case, the appeal of the framework is not merely analogical:
$\Theta_t$ is a measurable evolving kernel, and the framework predicts
constraints on its trajectory relative to information gain and energetic
expenditure during training.

\paragraph*{Empirical protocol (falsifiable).}
To test this section's claims, one can: (i) estimate $\Theta_t$ at fixed
training intervals for width-scaled model families, (ii) compute a proxy
for $\dot{I}_{\Theta_t}$ from held-out representation statistics,
and (iii) record wall-power draw to estimate $\dot{W}(t)$.
The prediction is an inequality trend
$\dot{W}(t) \geq c\,\dot{I}_{\Theta_t}$ up to calibration constant $c$,
with tighter agreement at larger width and slower learning rate.

\subsection{Adaptive Sampling for Dynamic Algal Blooms in a Lake}

Consider a lake with a time-varying bloom field
$b_t(x)$ (e.g., chlorophyll concentration) over spatial location
$x \in \Omega$. A natural operational realization is a heterogeneous team:
an autonomous surface vehicle (ASV) carrying an autonomous underwater
vehicle (AUV) with coordinated dock and undock. The ASV supplies
long-endurance transit, surface-visible fields (e.g., temperature and
surface color proxies), and a stable platform for docking, recharge, and
data offload; the AUV, when undocked, samples vertical structure (e.g.,
chlorophyll maximum depth, oxygen, turbidity) that the surface cannot see.
The key operational difficulty is that bloom fronts advect and deform on
timescales comparable to the mission duration, while subsurface structure
can misalign with surface patches---so the team's belief couples two
partially overlapping observation operators fused at rendezvous.

Let $\sigma_t \in \{\mathrm{docked},\mathrm{undocked}\}$ denote the
coordination state. While $\sigma_t=\mathrm{docked}$, the AUV is carried
and acts as payload; undocking initiates a dive phase with distinct motion
costs and information yield. Let $k_t$ be the kernel used by the onboard
model for spatiotemporal covariance in bloom dynamics (possibly including
depth or modality-specific components after vertical profiles are merged
at dock), and let $u_t(x)$ denote posterior uncertainty under $k_t$.
Actions decompose into ASV waypoints $a_t^{\mathrm{ASV}}$, AUV profile
commands when $\sigma_t=\mathrm{undocked}$, and feasible dock/undock
transitions; subsurface samples accrue only in the undocked phase.
Write $\mathcal{A}_t$ for the joint tuple
$(a_t^{\mathrm{ASV}},a_t^{\mathrm{AUV}},\sigma_t)$.
Mission resources impose constraints:
\begin{equation}
\begin{split}
    \sum_t \bigl( c_{\mathrm{move}}^{\mathrm{ASV}}(a_t^{\mathrm{ASV}},a_{t-1}^{\mathrm{ASV}})
    + c_{\mathrm{dive}}^{\mathrm{AUV}}(a_t^{\mathrm{AUV}},\sigma_t)
    \bigr) &\leq E_{\max}, \\
    \sum_t \mathbbm{1}[\text{collect at } t] &\leq N_{\max}.
\end{split}
\end{equation}
For sample return to base (or shore offload), a feasibility reserve must
be maintained; for heterogeneous operation, rendezvous between ASV and
surfacing AUV imposes an analogous coupled constraint on relative position
and time:
\begin{equation}
\begin{split}
    c_{\mathrm{return}}(a_t^{\mathrm{ASV}},\text{base})
    &\leq E_{\mathrm{reserve}}(t), \\
    c_{\mathrm{meet}}(a_t^{\mathrm{ASV}},a_t^{\mathrm{AUV}},\sigma_t)
    &\leq \Delta_{\mathrm{meet}}(t).
\end{split}
\end{equation}

The adaptive objective is to maximize expected information gain about
future bloom structure per mission cost:
\begin{equation}
\begin{split}
    \max_{\{\mathcal{A}_t\},\{k_t\}}
    \; \mathbb{E}\!\left[\sum_t \Delta I_t(k_t,\mathcal{A}_t)\right]
    &- \lambda_E \sum_t c_{\mathrm{move}}^{\mathrm{ASV}}(a_t^{\mathrm{ASV}},a_{t-1}^{\mathrm{ASV}}) \\
    &- \lambda_N \sum_t \mathbbm{1}[\text{collect at } t],
\end{split}
\end{equation}
where dive costs are
subsumed in the energy budget above, with dynamics over kernels governed
by the MaxCal tradeoff in Section~\ref{sec:maxcal_kernels}. Intuitively,
$k_t$ should reweight distinctions toward moving bloom boundaries and
shear-aligned filaments, and---after each dock---toward consistent
cross-depth covariance, because those structures yield maximal uncertainty
reduction per cycle subject to rendezvous feasibility.

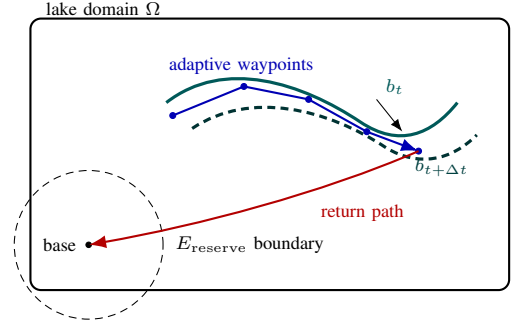
\begin{figure}[t]
\centering
\begin{tikzpicture}[>=Latex,scale=0.86,every node/.style={font=\scriptsize}]
    \draw[thick,rounded corners] (0,0) rectangle (7.4,4.2);
    \node[anchor=west] at (0.1,4.35) {lake domain $\Omega$};

    \filldraw[black] (0.9,0.7) circle (1.2pt);
    \node[anchor=east] at (0.85,0.7) {base};
    \draw[densely dashed] (0.9,0.7) circle (1.15);
    \node[anchor=west] at (2.1,0.7) {$E_{\mathrm{reserve}}$ boundary};

    \draw[very thick,teal!70!black]
        (2.1,2.9) .. controls (3.0,3.6) and (4.2,3.2) .. (5.1,2.6)
        .. controls (5.7,2.2) and (6.2,2.4) .. (6.6,2.9);
    \draw[very thick,teal!40!black,densely dashed]
        (2.5,2.5) .. controls (3.4,3.1) and (4.6,2.8) .. (5.5,2.2)
        .. controls (6.0,1.9) and (6.5,2.0) .. (6.9,2.4);
    \node[teal!70!black,anchor=west] at (5.35,3.15) {$b_t$};
    \node[teal!50!black,anchor=west] at (5.8,1.9) {$b_{t+\Delta t}$};
    \draw[->,thin] (5.35,2.95) -- (5.75,2.45);

    \filldraw[blue!70!black] (2.2,2.7) circle (1.3pt);
    \filldraw[blue!70!black] (3.3,3.15) circle (1.3pt);
    \filldraw[blue!70!black] (4.3,2.95) circle (1.3pt);
    \filldraw[blue!70!black] (5.2,2.45) circle (1.3pt);
    \filldraw[blue!70!black] (6.0,2.15) circle (1.3pt);
    \draw[->,blue!70!black,thick]
        (2.2,2.7) -- (3.3,3.15) -- (4.3,2.95) -- (5.2,2.45) -- (6.0,2.15);
    \node[blue!70!black,anchor=south west] at (2.0,3.2) {adaptive waypoints};

    \draw[->,red!70!black,thick] (6.0,2.15) .. controls (4.5,1.5) and (2.6,1.0) .. (0.9,0.7);
    \node[red!70!black,anchor=west] at (4.35,1.2) {return path};
\end{tikzpicture}
\caption{Lake algal-bloom scenario for adaptive sample return.
The bloom front advects from $b_t$ to $b_{t+\Delta t}$; adaptive-kernel
planning shifts waypoints toward moving high-information boundaries while
maintaining return-to-base feasibility. The same structure extends to an
ASV--AUV team with coordinated dock/undock: surface waypoints place the
stack for subsurface profiles, and rendezvous windows replace
return-to-base alone as a binding feasibility constraint.}
\label{fig:lake_bloom_adaptive_sampling}
\end{figure}

This case is operational rather than metaphorical because $k_t$ can
be interpreted as a task-specific similarity or acquisition function
governing what distinctions are worth sensing under budgeted action.
This yields a concrete adaptation argument: fixed kernels induce static
sampling lattices that under-sample moving bloom fronts, while kernel
adaptation concentrates trajectories along information-rich transients
subject to return-feasibility constraints. For ASV--AUV teams, undock
timing is an additional discrete decision: dives should occur when
subsurface uncertainty is high relative to surface-observable structure,
and heavy kernel or map updates can run on the ASV while docked.
\paragraph*{Testable prediction.}
Against a fixed-kernel baseline, adaptive-kernel planning should produce
higher forecast skill at equal energy budget and equal returned sample
count, with the largest gains during high-advection intervals; for
heterogeneous teams, gains should be most visible on depth-resolved
forecast metrics when surface and subsurface fields decorrelate.

\subsection{Biological Evolution}

Biological evolution instantiates kernel dynamics at the level of
perceptual and cognitive systems. The fitness kernel
$k_f(\text{genotype}, \text{genotype}')$ encodes which genetic
differences translate into phenotypic differences that selection can act
on. This kernel co-evolves with the genome under:
\begin{equation}
    \frac{dk_f}{dt} = -\lambda \frac{\delta \mathcal{F}_{\mathrm{fitness}}}
    {\delta k} + \eta(t)
\end{equation}
where $\mathcal{F}_{\mathrm{fitness}}$ is a free-energy-like fitness
functional and $\eta$ is mutational noise. Speciation events
correspond to bifurcations in kernel space---transitions between basins
of attraction of distinct self-consistent fitness kernels.
Operationally, $k_f$ can be estimated from genotype--phenotype--fitness
datasets by fitting local similarity operators that predict fitness
response to perturbations.

\subsection{Craft Mastery and Embodied Knowledge}

In~\cite{das2026navagunjara}, the design of a large-scale public artwork
was modeled as a product kernel problem: the final artifact is the
configuration that simultaneously satisfies the inductive kernels of all
contributors---mythological, craft, structural, and environmental.

The MaxCal kernel framework provides the dynamical substrate for this
observation.

A concrete illustration is provided by \emph{Navagunjara Reborn}, a large-scale Burning Man sculpture whose design and realization can be read as a trajectory through interacting kernels rather than execution of a fixed blueprint. In that project, the final artifact emerged from the simultaneous action of mythological, craft, structural, logistical, and environmental kernels: the lead artist's symbolic and site-specific priors, Rajesh Moharana's dhokra metalwork intuitions, Ekadashi Barik's cane-forming expertise, the engineering team's load-bearing and fabrication constraints, and the Black Rock Desert's wind, transport, and fire-safety requirements. The sculpture's realized form was not the maximizer of a single objective, but the configuration that remained viable under all of these distinction-making systems at once. In this sense, the project instantiates the product-kernel view dynamically: iterative digital modeling, photogrammetric feedback, structural redesign, material substitution, and on-playa improvisation can be interpreted as successive updates to a collective kernel trajectory, progressively narrowing the reachable design space until a coherent artifact became the least-assuming surviving path. At the level of individual practice, the master artisan's embodied kernel appears as a locally stable fixed point in morphological decision space; at the level of the collaboration, the completed sculpture is the transient intersection of several such fixed points under hard environmental and logistical constraints.

A master craftsman's kernel (e.g., Rajesh Moharana's dhokra
metalwork kernel encoding lost-wax bronze topology) is a stable
self-consistent kernel: the distinctions encoded by decades of practice
are self-reinforcing under the thermodynamic dynamics of craft
execution~\cite{das2026navagunjara}. Such a kernel is a fixed point because
the geometry it induces on morphological space makes it the
maximum-caliber choice given the constraints of the tradition.
In this domain, a measurable proxy for kernel evolution is a similarity
operator inferred from artifact morphology and process traces across an
apprenticeship time series.

Apprenticeship is a trajectory through kernel space from an unstable
initial kernel toward the master's fixed point. The thermodynamic cost
bound~(\ref{eq:kernel_landauer}) predicts that this trajectory requires
sustained metabolic investment proportional to the mutual-information
gap between novice and master kernels.

% -----------------------------------------------------------------------
\section{Relation to Assembly Theory}
\label{sec:assembly}
% -----------------------------------------------------------------------

Assembly theory~\cite{sharma2023} measures the assembly index $a(x)$
of an object $x$ as the minimum number of recursive joining steps
required to construct it from basic building blocks. The assembly index
is observer-independent and empirically measurable by mass spectrometry.

The relationship to MaxCal over kernels is:

\begin{enumerate}
    \item \textbf{Ontogenesis vs.\ mechanics.} Assembly theory describes
    how construction complexity accumulates until a system first acquires
    the capacity to distinguish microstates---to bear a kernel. MaxCal
    over kernels describes how that kernel evolves once born.

    \item \textbf{Assembly index as RKHS complexity.} We conjecture that
    the assembly index $a(x)$ is lower-bounded by the RKHS complexity
    of the kernel required to represent $x$:
    \[
        a(x) \geq c \cdot \| k_x \|_{\mathcal{H}} + O(1)
    \]
    where $k_x$ is the minimal kernel distinguishing $x$ from
    its chemical precursors.

    \item \textbf{The threshold $a(x) \geq 15$.} The empirical
    signature of selection in assembly theory---objects with
    $a(x) \geq 15$ are almost certainly products of
    selection~\cite{sharma2023}---corresponds, in our framework, to
    the minimum RKHS complexity required for a self-consistent kernel
    to exist. Below this threshold, the kernel space is too simple to
    support self-reinforcing fixed points.
\end{enumerate}

\subsection{Conceptual Synthesis}

The three frameworks compared here share information as a central
concept and differ primarily in their primitive variables:

\begin{table}[h]
\centering
\caption{Framework primitives and key observables.}
\label{tab:triangle}
\begin{tabular}{lll}
\toprule
\textbf{Framework} & \textbf{Primitive} & \textbf{Key observable} \\
\midrule
This work        & Path over $\mathcal{K}$     & Kernel trajectory \\
Static kernel limit & Fixed $k(x,x')$          & Mutual information $I$ \\
Assembly theory  & Construction path           & Assembly index \\
Maximum caliber  & Path entropy                & Trajectory distribution \\
\bottomrule
\end{tabular}
\end{table}

The present paper proposes a synthesis in which kernel dynamics describes
how kernels traverse the landscape between chemical origin constraints
(assembly theory) and stable operating points (static-kernel limit), along
least-assuming trajectories selected by MaxCal.

% -----------------------------------------------------------------------
\section{Claim Levels and Scope}
\label{sec:claim_levels}
% -----------------------------------------------------------------------

For clarity, we separate three claim types used in this paper.

\begin{enumerate}
    \item \textbf{Formal definitions/results in this manuscript.}
    Kernel space $\mathcal{K}$, MaxCal lifting to path measures on
    $\Pi$, and the fixed-point condition for self-consistent kernels.

    \item \textbf{Structured correspondences.}
    Mappings from the core framework to RG and finite-width NTK evolution.
    These are mathematically motivated identifications, but not complete
    equivalence proofs in the present draft.

    \item \textbf{Conjectural bridges.}
    Biological niches, craft mastery, adaptive field sampling, and
    assembly-theoretic thresholds are proposed as testable interpretations
    that require dedicated empirical and model-specific validation.
\end{enumerate}

The intended contribution is therefore a \emph{unifying variational
framework with explicit testable conjectures}, not a completed final theory
of order across all domains.

% -----------------------------------------------------------------------
\section{Open Questions}
\label{sec:open}
% -----------------------------------------------------------------------

\begin{enumerate}

    \item \textbf{Kernel geodesics.} What is the geodesic in
    $(\mathcal{K}, d_{\mathrm{HS}})$ between two self-consistent
    kernels, and does it pass through a saddle point corresponding to a
    Kuhnian crisis?

    \item \textbf{Quantum kernel dynamics.} Does the Fubini-Study metric
    extension to $\mathbb{CP}^{n-1}$ support an analogous MaxCal
    formulation for quantum kernels, and do quantum phase transitions
    correspond to fixed-point bifurcations in quantum kernel space?

    \item \textbf{Assembly index and RKHS complexity.} Can the conjectured
    bound $a(x) \geq c \| k_x \|_{\mathcal{H}}$ be proved or
    disproved for a specific model system (e.g., small organic molecules)?

    \item \textbf{NTK convergence.} Does the NTK of a trained diffusion
    model converge to the Hellinger kernel $k_H$ of the data
    distribution in the infinite-width limit?

    \item \textbf{Lake bloom adaptive sampling.} Can an online
    kernel-adaptation policy for dynamic algal blooms improve
    chlorophyll-forecast skill at fixed energy and fixed returned-sample
    count relative to fixed-kernel sampling---including for heterogeneous
    ASV--AUV missions with coordinated dock/undock and depth-resolved
    validation?

    \item \textbf{Paradigm shift thermodynamics.} Can the thermodynamic
    cost bound~(\ref{eq:kernel_landauer}) be used to predict the
    timescale of scientific paradigm shifts from citation network
    data---testing the Kuhn corollary as a quantitative prediction?
    A measurable proxy is a time-indexed similarity kernel over papers
    (e.g., embedding or co-citation based) and its trajectory in
    Hilbert-Schmidt distance.

\end{enumerate}

% -----------------------------------------------------------------------
\section{Conclusion}
% -----------------------------------------------------------------------

We have proposed Maximum Caliber over kernel space as a variational
principle for the dynamics of distinction-making systems. The framework
organizes renormalization group flow, neural tangent kernel evolution,
biological speciation, and craft mastery as candidate instantiations and
analogues of a single kernel-dynamics picture. Fixed points are
self-consistent kernels whose stability is set by the curvature of the
frozen-kernel path-entropy objective. Kernel change has a thermodynamic
cost under explicit information-thermodynamic assumptions, motivating
quantitative predictions for paradigm-transition timescales and adaptive
sampling policies.

Together with assembly theory~\cite{sharma2023} and
Maximum Caliber~\cite{press2013}, kernel dynamics supports a
three-part research program: assembly theory explains how the first
kernel-bearing systems arise; kernel dynamics explains how those kernels
evolve; and the fixed-kernel limit analyzed here explains what they do
while approximately stable, with MaxCal providing the variational
principle connecting scales.

The present paper is intended not as a completed theory of kernel change
across all domains, but as a variational framework that makes such change
mathematically discussable and empirically targetable. Its value lies in
isolating a common question across learning, physics, biology, and craft:
what governs trajectories through the space of distinction-making structures
when representational change itself becomes the dynamical variable?

\section*{Acknowledgments}
The framework builds on intellectual proximity to Sara Imari Walker's
program on assembly theory and Steve Press\'e's work on maximum caliber,
both at Arizona State University. The craft-kernel instantiation
developed from collaboration with Rajesh Moharana, Ekadashi Barik, and
the artisan teams of Odisha documented in~\cite{das2026navagunjara}.

\end{document}